 \documentclass[iicol]{sn-jnl}
\usepackage[T1]{fontenc}
\usepackage{graphicx}
\graphicspath{images/}
\usepackage[algo2e,linesnumbered,ruled,lined]{algorithm2e}
\usepackage{amsfonts}
\usepackage{multirow}
\graphicspath{{images/}}
\usepackage{amsmath}
\usepackage{amsthm}

\DeclareMathOperator*{\argmin}{arg\,min}
\usepackage{xcolor}
\usepackage{psfragx}

\usepackage[noabbrev]{cleveref}
\newcommand{\esp}{\mathbb{E}}
\newcommand{\re}{\mathcal{R}e}

\newtheorem{theorem}{Theorem}[section]

\newtheorem{lemma}[theorem]{Lemma}

\begin{document}
\title[Comp. Learn. deep reg.]{Batch-less stochastic gradient descent for compressive learning of deep regularization for image denoising}

%
%

\author[1]{\fnm{Hui} \sur{Shi}}\email{hui.shi@u-bordeaux.fr}

\author*[1]{\fnm{Yann} \sur{Traonmilin}}\email{yann.traonmilin@math.u-bordeaux.fr}

\author[1]{\fnm{Jean-François} \sur{Aujol}}\email{jean-francois.aujol@math.u-bordeaux.fr}

\affil[1]{\orgname{Univ. Bordeaux, Bordeaux INP, CNRS,  IMB, UMR 5251}, \postcode{F-33400}, \city{Talence}, \country{France}}

 \abstract{We consider the problem of denoising  with the help of prior information taken from a database of clean signals or images. Denoising with variational methods is very efficient if a regularizer well adapted to the nature of the data is available. Thanks to the maximum a posteriori Bayesian framework, such regularizer can be systematically linked with the distribution of the data. With deep neural networks (DNN), complex distributions can be recovered from a large training database.
To reduce the computational burden of this task,  we  adapt the compressive learning framework to the learning of regularizers parametrized by DNN. We propose two variants of stochastic gradient descent (SGD) for the recovery of deep regularization parameters from a heavily compressed database. These algorithms outperform the initially proposed method that was limited to low-dimensional signals, each iteration using information from the \emph{whole database}. They also benefit from classical SGD convergence guarantees. Thanks to these improvements we show that this method can be applied for patch based image denoising.} 
%

%
\maketitle

\section{Introduction}
At the heart of imaging inverse problems,  having a precise prior information on the distribution of the unknown image is crucial for efficient recovery of said image. In particular, consider the  denoising problem, i.e. finding an accurate estimate $u^\star$ of the original image $u_0 \in \mathbb{R}^d$ from the observed noisy image $v\in \mathbb{R}^d$:
\begin{equation}
v = u_0 + \epsilon,
\end{equation}
where the noise $\epsilon$ (assumed to be additive white Gaussian noise of standard deviation $\sigma$) is independent of $u_0$.
Recovering  $u_0$ from its degraded version $v$ is an ill-posed problem and we need to use additional (prior) information about the unknown image $u_0$ to obtain meaningful solutions.
A common strategy \cite{demoment1989image} for solving inverse problems is to define an estimator $u^\star$ which is the minimizer of a functional:
\begin{equation}
u^\star \in \argmin_u F(u) + \lambda R(u),
\end{equation}
where $F$ is the data fidelity term making the solution consistent with the observation $y$ and $R$ is the regularization term that incorporates the prior information, weighted by the regularization parameter $\lambda > 0$.
The choice of $R$ depends on the statistics of the signal of interest which is not always available in real-life applications.

The maximum a posteriori (MAP) Bayesian framework provides a useful tool to interpret such methods.
The MAP estimator is given by:
\begin{equation}\label{eq:denoising}
u^\star_{\text{MAP}}
\in \argmin_u \| v-u \|_2^2 - \lambda \log (\mu(u))
\end{equation}
where $\mu $ denotes a prior probability law (of density $\mu(\cdot)$) of the unknown data $u$. In this context, the regularizer is related to the prior distribution of the data, i.e.,  $R(u)= -log(\mu(u)) $.

Defining an  accurate prior model in the form of a regularizer or distribution of the images of interest is one of the main difficulties for designing efficient estimation methods.
Classical Bayesian approaches, e.g. in image processing, rely on explicit priors such as total variation or Gaussian mixture models (GMM)~\cite{zoran2011learning} trained on a database of image patches.
Recently, researchers proposed to use DNN to design the regularizer. Methods such as the total deep variation~\cite{kobler2021total},  adversarial regularizers~\cite{lunz2018adversarial,prost2021learning}, as well as the Plug \& Play approach and its extensions \cite{venkatakrishnan2013plug,hurault2022gradient} deliver remarkably accurate results.
However, such models are typically learned from large datasets. Estimating their parameters from such a large-scale dataset is a serious computational challenge.
\paragraph{Compressive learning}
One possibility to reduce the computational resources of learning consists in using the compressive learning (CL) framework~\cite{gribonval2021compressive,gribonval2021statistical,gribonval2021sketching}.
The main idea of CL, coined as sketching, is to compress the whole data collection into a fixed-size representation, a so-called \textit{sketch} of data, such that enough information relevant to the considered learning task is captured. Then the learned parameters are estimated by minimizing a non-linear least-square problem built with the sketch.
The size of the sketch $m$ is chosen proportional to the intrinsic complexity of the learning task. Consequently, the cost of inferring the parameters of interest from the sketch does not depend on the size of the training database but on the number of parameters we want to estimate. Hence, it is possible to exploit arbitrarily large datasets in the sketching framework without demanding more computational resources.

During the sketching phase, a huge collection of $n$ $d$-dimensional data vectors $X = \{x_i\}_{i = 1}^n$ is summarized into a single $m$-dimensional $( m \ll n)$ vector $\hat{z}$ with:
\begin{equation}\label{eq: sk_emp}
\hat{z} = \frac{1}{n} \sum_{i = 1}^n \Phi(x_i) = \mathcal{S} (\hat{\mu}_n),
\end{equation} 
where $\hat{\mu}_n := \frac{1}{n} \sum_{i = 1}^n \delta_{x_i}$ is the empirical probability distribution of the data,  $\delta_{x_i}$ is the Dirac measure at $x_i$ and the function $\Phi : \mathbb{R}^d \to \mathbb{C}^m$ is called the feature map (typically random Fourier moments). The sketch $\hat{z}$ is the mean of the feature map over the whole database.

Sketching can be interpreted as a linear operation
$\mathcal{S}$  on measures $\mu$  defined by
$
\mathcal{S}\mu:= \mathbb{E}_{X \sim \mu} \Phi(X).
$
An estimate of a distribution $\mu_\theta$ (or of distributional parameters $\theta$ of interest) can be  calculated from the sketch  by solving:
\begin{equation}\label{eq: sk matching}
\mu ^\star_\theta \in \argmin_{\mu_\theta} G(\theta) = \argmin_{\mu_\theta} \| \hat{z} - \mathcal{S}\mu_\theta \| _2^2.
\end{equation}
In practice, this "sketch matching" problem can be solved by greedy compressive learning Orthogonal Matching Pursuit (OMP) algorithm and its extension Compressive Learning-OMP with replacement \cite{keriven2018sketching} when  $\mu_\theta$ is a mixture of elementary distribution (i.e. a GMM). When the distribution $\mu$ is a GMM in high-dimension with flat tail covariances, the problem can also be solved by the Low-Rank OMP algorithm (LR-OMP). It was shown that the prior model learned
with LR-OMP can be used to perform image denoising~\cite{shi2022sketching} with no loss of performance compared to the non compressive approach, and with  faster training time.

These greedy algorithms are suitable for any sketching operator $\mathcal{S}$ and any distribution density $\mu$, as long as the sketch $\mathcal{S}\mu$ and its gradient $\nabla_\theta  \mathcal{S}\mu$ with respect to the distributional parameters $\theta$ of interest have a closed-form expression: the core of these OMP-based algorithms is computing the expression of $\mathcal{S}\mu$ and $\nabla_\theta  \mathcal{S}\mu$. However, for some more general distributions, the sketching feature map  may not have a closed form. This limits the  use of the sketching framework.

In this paper, our goal is to recover a good approximation of the probability distribution of any unknown data from its sketch (i.e. beyond GMM). As neural networks (NN) have great expressive power \cite{hornik1989multilayer,pan2016expressiveness}, we propose to tackle the problem by adapting the sketching to NN. More precisely, we propose to define the regularizer $R_\theta$ parametrized by a DNN $f_\theta$ (precisely a ReLU network) as
\begin{equation}\label{eq: design}
R_\theta(\cdot) = \|f_\theta(\cdot) \|_2^2.
\end{equation}
Such a regularization corresponds to the parametric distribution density $\mu_\theta  \propto e^{-\|f_\theta (\cdot)\|_2^2}$. Thus it can be viewed as a generalized Gaussian distribution, where the bilinear form induced by the covariance matrix is replaced by a network. Due to the fact that NN have good generalization properties,
the proposed regularization should be capable of encoding complex probability distributions. Unfortunately, a direct practical application of existing tools is not possible as closed-form expressions of $\mathcal{S} \mu$ are not available for sketching operator $\mathcal{S}$ based on random Fourier features.  In  \cite{shi2023compressive}, it was shown on low-dimensional  (2D and 3D) data that it was possible to estimate a deep regularizer by approximating $\mathcal{S}$ by a sampled version on a regular grid  $\mathcal{S}_d$. Unfortunately the use of a grid limits the extension of this method to data in higher dimension such as image patches.

\paragraph{Contributions and outline}
In this work, we propose novel approaches to learn deep regularizers from sketches beyond our initial method from \cite{shi2023compressive}.
Instead of relying on a grid-based discretization of the sketching operator $\mathcal{S}$, we propose an adaptation of the stochastic gradient descent method (that we call compressive learning stochastic gradient descent, CL-SGD), dynamically generating descent directions from the \emph{whole training dataset}  with the help of  a  random discretization of the sketching operator performed at each iteration.
This strategy not only makes the approach suitable for higher-dimensional problems but also substantially enhances the efficiency and flexibility of the sketching process, requiring far fewer grid points and delivering considerably faster results compared to its predecessor. Once the neural network is trained denoising can be performed using classical variational methods.
The method is tested using both synthetic and real  audio and image data, demonstrating that a deep prior can be learned to perform patch based denoising.
Moreover we provide a theoretical analysis ensuring the convergence of our compressive learning stochastic gradient method.

The rest of this article is organized as follows. We start by introducing the sketching framework,  ReLU networks and some related works in \cref{sec: 2}. In \cref{sec: 3}, we describe the proposed framework: the adaptation of the compressive learning framework to the learning of regularizers parametrized by ReLU networks.
\Cref{sec: 5} illustrates the performance of the proposed methods on both synthetic data and real-life data. Finally, conclusions are drawn in \cref{sec: 6}.

\section{Background, related works}\label{sec: 2}
We suppose that data samples $x_i$ are modeled as independent and identically distributed random vectors having an unknown probability distribution with density $\mu \in \mathcal{M}(\mathcal{D})$ (where $\mathcal{M}(\mathcal{D})$ is the set of measures having a (Gateaux)-differentiable  density supported  on a domain $ \mathcal{D}\subset \mathbb{R}^d$). For simplicity we identify $\mu$ with its density $\mu : \mathbb{R}^d \to \mathbb{R}^+$. Hence we can evaluate $\mu$ on any point $p_i \in \mathbb{R}^d $ with $\mu(p_i)$. For a collection of
points $\mathbf{p}= (p_i)_{i=1}^P$, we write $\mu(\mathbf{p}) = (\mu(p_i))_{i=1}^P \in \mathbb{R}^P$.
We define the linear sketching operator $\mathcal{S}$ that maps
$\mu$ to the $m$-dimensional sketch vector $z$:
\begin{equation}
\begin{split}
\mathcal{S} : \;& \mathcal{M}(\mathcal{D})\rightarrow \mathbb{C}^m \\
& z = \mathcal{S} \mu:= \int_{ \mathbb{R}^d} \mu(x) \Phi(x)dx .
\end{split}
\end{equation} 
When the transformation (sketching feature map) $\Phi(\cdot)$ is built with random frequencies of the Fourier transform, 
the $l$-th component of the sketch is
\begin{equation}\label{eq:sk component}
z_l = \int_{ \mathbb{R}^d} e ^{- j \langle \omega_l, x \rangle}\mu(x) dx, \quad \text{for} \quad l = 1, \dots , m,
\end{equation}
where $\{\omega_l\}_{l = 1}^m \in \mathbb{R}^d$ are
frequencies drawn at random.
Taking a statistical perspective, the components $z_l$ can be seen as samples of the characteristic function of $\mu_{\mathcal{X}}$. 
Accordingly, given a dataset $X = \{x_i \}_{i=1}^{n}$, the empirical sketch $\hat{z}$ can be computed from the samples of the database as
\begin{equation}
\hat{z}_l = \frac{1}{n} \sum_{i = 1}^n e^{-j \langle\omega_l, x_i \rangle}  , \quad \text{for} \quad l = 1, \dots , m.
\end{equation}
The compression ratio $r$ is $m/nd$.
It was shown \cite{keriven2018sketching,gribonval2020sketching,gribonval2021statistical} that when the probability distribution $\mu$ has a low dimension structure, e.g. a GMM, one can recover it (with high probability) from enough randomly chosen samples of its Fourier transform. The required size of the sketch is typically  of the order of the number of parameters we need to estimate.
\paragraph{ReLU network}
A ReLU network, denoted by $f_\theta$, is defined as a fully connected, feed-forward network (multi-layer perceptrons) with rectified linear unit (ReLU) activations. This activation has grown in popularity in feed-forward networks due to the success of first-order gradient based heuristic algorithms and the improvement in convergence to the approximated function for training \cite{nair2010rectified}.
\paragraph{Related works}
The sketching framework has been successfully applied to parametric models including GMMs \cite{keriven2018sketching,gribonval2021statistical,shi2022sketching}, K-means clustering \cite{keriven2018sketching,gribonval2021statistical} and classification \cite{schellekens2018compressive}.
These methods are limited to the models for which the sketch function has a closed form.
In our work, we apply the sketching to neural networks to encode more complex and high-dimensional probability distributions.
Sketching techniques have found applications in neural networks, as evidenced by their use in \cite{schellekens2020compressive} and \cite{shi2023compressive}. In \cite{schellekens2020compressive}, the integration of sketching with generative networks enables the generation of data samples, while \cite{shi2023compressive} pursues the goal of developing a deep regularizer for solving inverse problems.
Notably, in \cite{schellekens2020compressive}, the authors suggest an approximation of the sketching map using Monte-Carlo sampling, whereas in \cite{shi2023compressive}, we  chose for a discrete sketching operator for the approximation. It is worth highlighting that the sketching framework outlined earlier emphasizes data-independent approximation, specifically obtaining sketches through the averaging of random features.



\section{Proposed method}\label{sec: 3}
In this section, we  explain how we adapt the sketching framework to estimate regularizations by DNN. We then provide a theoretical analysis of our method and we detail the implementation of our algorithms.
\subsection{Previous work}
We start by explaining why there are no explicit closed-form expressions of the sketching function available in the context of prior parametrized by DNN. Intuitively, since ReLU networks define piecewise affine functions, we can indeed express a ReLU network $f_\theta$ as:
\begin{equation}
f_\theta(x)= \sum_{\gamma=1}^{N_R} \mathbf{1}_{R_\gamma}(x)(W_\gamma x +b_\gamma),
\end{equation}
where $\mathbf{1}_{R_\gamma}$ is the indicator function of each of the $N_R$ affine regions $R_\gamma$, with parameters ($W_\gamma,b_\gamma$).


Given a dataset $X$, we aim at learning, from only the sketch $\hat{z}$, an approximation $\mu_{\theta^*}$ of the probability distribution $\mu$ generating $X$.
We consider a regularizer of the form $R_\theta(\cdot) = \| f_\theta(\cdot) \|_2^2$ which corresponds to parametric densities of the form $\mu_\theta(\cdot) \propto e^{- R_\theta(\cdot)}$.
Ideally, with the definition \eqref{eq:sk component}, the sketch would have to be calculated as
\begin{equation}\label{eq:sk}
\begin{split}
z_l 
&= \int_{\mathbb{R}^d} e^{-j\langle \omega, x \rangle }e^{-\|f_\theta(x) \|_2^2} dx\\
&= \int_{\mathbb{R}^d}
 e^{-j\langle \omega, x \rangle }e^{- \sum_{p = 1}^d (\sum_\gamma^{N_R} \mathbf{1}_{R_\gamma}(x)((W_\gamma x) +{b_\gamma}))^2} dx .\\
\end{split}
\end{equation}
This would require Fourier Transforms on the individual regions $R_\gamma$.
However, to the best of our knowledge, there is no analytic expression of such Fourier transform (Fourier transform on polygons).

In our previous work~\cite{shi2023compressive}, it was proposed to  perform the following minimization
\begin{equation}\label{eq_Sd}
\theta^* \in \argmin_{\theta \in \Theta} \left\|\mathcal{S}_{\mathbf{p}}\mu_\theta  - \hat{z} \right \|_2^2,
\end{equation}
where $\mathcal{S}_\mathbf{p}$ is a discretization on a regular grid  ${\mathbf{p}} = (p_i)_{i=1}^P \subset \mathbb{R}^d$ over the domain of the data, i.e.
\begin{equation}
 (\mathcal{S}_\mathbf{p} \mu_\theta)_l \propto \sum_{p_i = 1}^P e^{-j \langle \omega_l,p_i\rangle} \mu_\theta (p_i)
\end{equation}
This optimization problem is  then solved through gradient descent based methods.
Considering the data at iteration $t$ as $\theta^t$, the update step can be expressed as follows:
\begin{equation}
\theta^{t} = \theta^{t-1} - \eta 2 \re \left( ( \nabla \mathcal{S}_\mathbf{p} \mu_\theta )^* \left(\mathcal{S}_{\mathbf{p}}\mu_\theta - \hat{z}
\right)\right)
\label{eq_init_method}
\end{equation}
where $\eta > 0$ represents the learning rate.

With the discretization, if $\mu_\theta$ is differentiable at point $p_i$, the gradient of $\mathcal{S}_p\mu_{\theta}$ with respect to the parameters $\theta$ can be computed easily by
using the automatic differentiation. Note that the discretization is used only in the estimation of the regularizer from the sketch. It thus only impacts the calculation time and memory requirement of the estimation of the regularizer and not the size of the compressed dataset itself.
Of course, the major pitfall of this approximation is the limitation for applications in high dimension as the number of points is exponential with respect to the dimension $d$. The required boundedness (or approximate boundedness such as in the Gaussian case) of the data is a valid assumption in many practical applications in signal and image processing.
\subsection{Compressive learning stochastic gradient descent (CL-SGD)}\label{sec:pres_clsgd}
Instead of discretizing the forward operator $\mathcal{S}$ once on a grid, we propose to perform a stochastic gradient descent where descent directions are generated with the help of  a different random uniform discretization of  $\mathcal{S}$ at each step. We first introduce a naïve discretization which we then adadapt to a CL-SGD method with theoretical convergence guarantees.

As an approximation of  the minimization of $G$, first consider the discretized operator on a random grid ${\mathbf{p}}$ where $p_i \sim \mathcal{U}(\mathcal{D})$ ($\mathcal{U}(\mathcal{D})$ is the uniform distribution on  $\mathcal{D}$) and the corresponding function
\begin{equation}
\begin{split}
 &G_\mathbf{p} (\theta) = \left\| \mathcal{S}_\mathbf{p} \mu_\theta  - \hat{z} \right \|_2^2.
 \end{split}
\end{equation}
We remark that $\mathcal{S}_\mathbf{p} \mu_\theta$ can be written as 
\begin{equation}
\begin{split}
 \mathcal{S}_\mathbf{p} \mu_\theta = B_\mathbf{p} \mu_\theta(\mathbf{p})
 \end{split}
\end{equation}
where $\mu_\theta(\mathbf{p})$ is the density of $\mu_\theta$ evaluated on the grid $\mathbf{p}$ and where $B_{\mathbf{p},l,i} = \frac{e^{-j \langle\omega_l,p_i \rangle}}{P}$. It is shown in Section~\ref{sec:consistency}  that $\esp B_\mathbf{p}(\mu_\theta(\mathbf{p}) )= \mathcal{S}\mu_\theta$, i.e. this random discretization is consistent with the sketch in expectation. Consider the minimization of $G_\mathbf{p}$
\begin{equation}
\begin{split}
 \min_{\theta\in \Theta}\left\| B_\mathbf{p} \mu_\theta(\mathbf{p}) - \hat{z} \right \|_2^2.
 \end{split}
\end{equation}
where $\theta \subset \mathbb{R}^{d_0}$ is the set where the parameters (weights and bias) of the DNN live.
This is a simple non-linear least square problem. We calculate the directional derivative of this functional with respect to $\theta$ in a direction $h$ (the gradient is the evaluation of these derivatives in the directions formed by the canonical basis of $\mathbb{R}^{d_0}$):

\begin{equation}
\begin{split}
 \frac{1}{2} \partial_h G_\mathbf{p} (\theta)  &=  \re \langle  B_\mathbf{p} \partial_h \mu_{\theta}(\mathbf{p}), B_\mathbf{p}\mu_{\theta}(\mathbf{p}) -\hat{z}\rangle  \\
  \end{split}
\end{equation}

Recall that the directional derivatives of $G$ are given by

\begin{equation}
\begin{split}
 \frac{1}{2} \partial_h G (\theta) &= \re \langle \mathcal{S}  \partial_h \mu_{\theta}, \mathcal{S}\mu_{\theta} -\hat{z} \rangle.\\
   \end{split}
\end{equation}
Recall that we  supposed that $\mu_{\theta}$ is differentiable with respect to $\theta$ in any direction $h$ (usually referred as  Gateaux differentiability).
We  link the expectation of these derivatives with the directional derivatives of the original sketch matching functional $G$ with the following Lemma.
\begin{lemma}\label{lem:exp_deriv_naive}
Consider  $\mathcal{S}$ constructed with frequencies $(\omega_l)_{l=1}^m$.  Let $\mathbf{p} = (p_i)_{i=1}^P$ with $p_i \in \mathcal{U}(\mathcal{D})$, $\mu_\theta \in \mathcal{M}(\mathcal{D})$ and $h \in \mathbb{R}^{d_0}$ such that $\|h\|_2=1$. Then,
 \begin{equation}
 \begin{split}
 \esp_\mathbf{p} \partial_h G_\mathbf{p} (\theta)   &=\partial_h G (\theta) \\
    &+\frac{2}{P}   (m\langle \partial_h \mu_{\theta}, \mu_{\theta} \rangle_{L^2(\mathcal{D})} -\langle \mathcal{S}\partial_h \mu_{\theta},  \mathcal{S}\mu_{\theta} \rangle). \\
    \end{split}
 \end{equation}

\end{lemma}
\begin{proof}
Thanks to Lemma~\ref{lem:consistency2}, we calculate the expectation.
 \begin{equation}
\begin{split}
 & \frac{1}{2}\esp_\mathbf{p} \partial_h H (\theta) =   \esp_\mathbf{p} \langle  B_\mathbf{p} \partial_h \mu_{\theta}, B_\mathbf{p}\mu_{\theta} -\hat{z}\rangle \\
  &=  \esp \langle  B_\mathbf{p} \partial_h \mu_{\theta}, B_\mathbf{p}\mu_{\theta}\rangle - \langle \mathcal{S}  \partial_h \mu_{\theta}, \hat{z} \rangle\\
    &=   \frac{m}{P}   \langle \mu_1 , \mu_2 \rangle_{L^2(\mathcal{D})}  +\frac{P-1}{P}\langle \mathcal{S}\mu_1,  \mathcal{S}\mu_2 \rangle \\
    &- \langle \mathcal{S}  \partial_h \mu_{\theta}, \hat{z} \rangle \\
    &=\langle \mathcal{S}  \partial_h \mu_{\theta}, \mathcal{S}\mu_{\theta} -\hat{z} \rangle \\
    &+\frac{1}{P}   (m\langle \partial_h \mu_{\theta_n}, \mu_{\theta} \rangle_{L^2(\mathcal{D})} -\langle \mathcal{S}\partial_h \mu_{\theta},  \mathcal{S}\mu_{\theta} \rangle). \\
\end{split}
 \end{equation}
\end{proof}
Remark that, unfortunately, there is a bias term that converges to $0$ when $P \to \infty$. Hence this method will permit to obtain an approximation of the true gradients of $G$ when $P$ is large enough (note that the quality of the approximation does not depend on the size of the training database).  The corresponding theoretical naïve SGD algorithm is given in Algorithm~\ref{alg:meth1}.

\begin{algorithm2e}[t]\label{alg:meth1}
\caption{Naive CL-SGD algorithm}
\KwData{Sketch $\hat{z}$, $K$ iterations, $\gamma$ step size}
 $z_0 \gets \hat{z}$\;
 $\theta_0$ randomly initialized\;
 \For{$k = 0,..., K$}{Generate a grid $\mathbf{p}$ of $P$ points $p_i \sim \mathcal{U}(\mathcal{D})$\;
 $\theta_{k+1} = \theta_k - \gamma \nabla G_\mathbf{p}(\theta_k,z_k, p)$\;}
\end{algorithm2e}

Looking at the origin of the bias term in the proof of Lemma~\ref{lem:exp_deriv_naive}, we are able to propose another approximation of the directional derivative,  which is unbiased. Using two i.i.d  random grids $\mathbf{p}$ and $\mathbf{q}$, we define the directions
\begin{equation}
 d_{h,\mathbf{p},\mathbf{q}} =2\langle  B_\mathbf{p} \partial_h \mu_{\theta}(\mathbf{p}), B_\mathbf{q}\mu_{\theta}(\mathbf{q}) -\hat{z}\rangle
\end{equation}
Its expectation is exactly the gradient of the sketch matching problem.
\begin{lemma}\label{lem:partial_meth2}
Consider  $\mathcal{S}$ constructed with frequencies $(\omega_l)_{l=1}^m$.  Let $\mathbf{p} = (p_i)_{i=1}^P,\mathbf{q} = (q_i)_{i=1}^P$ with $p_i,q_i \in \mathcal{U}(\mathcal{D})$ i.i.d., $\mu_\theta \in \mathcal{M}(\mathcal{D})$ and $h \in \mathbb{R}^{d_0}$ such that $\|h\|_2=1$, then
 \begin{equation}
 \begin{split}
 \esp_{\mathbf{p},\mathbf{q}}  \left(d_{h,\mathbf{p},\mathbf{q}}\right)  &=\partial_h G (\theta). \\
    \end{split}
 \end{equation}

\end{lemma}
\begin{proof}
We calculate the expectation.
 \begin{equation}
\begin{split}
 &\esp_{\mathbf{p},\mathbf{q}}  (d_{h,\mathbf{p},\mathbf{q}})=   \esp_{\mathbf{p},\mathbf{q}}\langle  B_\mathbf{p} \partial_h \mu_{\theta}(\mathbf{p}), B_\mathbf{q}\mu_{\theta}(\mathbf{q}) -\hat{z}\rangle \\
\end{split}
 \end{equation}
 As $\mathbf{p},\mathbf{q}$ are i.i.d we have
 \begin{equation}
\begin{split}
 &\esp_{\mathbf{p},\mathbf{q}}  d_{h,p,\mathbf{q}}=  \langle  \esp_{\mathbf{p}} B_\mathbf{p} \partial_h \mu_{\theta}(\mathbf{p}), \esp_{\mathbf{q}} B_\mathbf{q}\mu_{\theta}(\mathbf{q}) -\hat{z}\rangle \\
\end{split}
 \end{equation}
 Using Lemma~\ref{lem:consistency} we have $\esp_{\mathbf{p}}\left( B_\mathbf{p} \partial_h \mu_{\theta}(\mathbf{p})\right) =\mathcal{S} \partial_h \mu_{\theta}$ and $\esp_{\mathbf{q}}  \left( B_\mathbf{\mathbf{q}}\mu_{\theta}(\mathbf{q}) \right) =\mathcal{S}  \mu_{\theta}$. This gives
  \begin{equation}
\begin{split}
 &\esp_{\mathbf{p},\mathbf{q}}  d_{h,\mathbf{p},\mathbf{q}}=\langle \mathcal{S}  \partial_h \mu_{\theta}, \mathcal{S}\mu_{\theta} -\hat{z} \rangle .\\
\end{split}
 \end{equation}
 \end{proof}

 With this expression we can build a stochastic descent direction with  expectation being exactly the gradient of $G$:

 \begin{equation}
  D_{\mathbf{p},\mathbf{q}} = (d_{e_i,\mathbf{p},\mathbf{q}})_{i=1}^{d_0}
 \end{equation}
 where $e_i$ are the elements of the canonical basis of $\mathbb{R}^{d_0} \supset \Theta$. We have
\begin{lemma}\label{lem:gradient_meth2}
Under the hypotheses of Lemma~\ref{lem:partial_meth2}, we have
 \begin{equation}
 \begin{split}
 \esp_{\mathbf{p},\mathbf{q}} \left( D_{\mathbf{p},\mathbf{q}}\right) &= \nabla G (\theta). \\
    \end{split}
 \end{equation}
\end{lemma}
\begin{proof}
 Use   Lemma~\ref{lem:partial_meth2}, for $h =e_i$ where the $e_i$ form the canonical basis of $\mathbb{R}^{d_0}$.
\end{proof}
We also have that a direct application of Lemma~\ref{lem:variance} permits to bound the variance of this estimator of $\nabla G (\theta) $:

\begin{equation}
\esp_{\mathbf{p},\mathbf{q}} \| D_{\mathbf{p},\mathbf{q}}  -\nabla G (\theta)\|_2^2 = O\left(\frac{1}{P}\right)
\end{equation}

This leads to the second theoretical method Algorithm~\ref{alg:meth2}.

\begin{algorithm2e}[t]\label{alg:meth2}
\caption{Unbiased CL-SGD algorithm}
\KwData{Sketch $\hat{z}$, $K$ iterations, $\tau$ step size}
 $z_0 \gets \hat{z}$\;
 $\theta_0$ randomly initialized\;
 \For{$k = 0,..., K$}{Generate a grid $\mathbf{p}$ of $P$ points $p_i \sim \mathcal{U}(\mathcal{D})$\;
 Generate a grid $\mathbf{q}$ of $P$ points $q_i \sim \mathcal{U}(\mathcal{D})$\;
 $\theta_{k+1} = \theta_k - \tau D_{\mathbf{p},\mathbf{q}}$\;}
\end{algorithm2e}

\subsection{Practical implementation of CL-SGD}
In practice, we implement the theoretical Algorithms \ref{alg:meth1} and \ref{alg:meth2} wih slight differences.

For  Algorithm~\ref{alg:meth1},we estimate the normalization constant of $\mu_{\theta}$ that best matches  $\mu_{\theta}$  to the sketch in the least square sense at the current iteration i.e. we minimize for $\alpha_p$ the objective $\|B_{\mathbf{p}}\alpha\mu_{\theta}-\hat{z}\|_2^2$. This gives the update
\begin{equation}
\alpha_{\mathbf{p}} =   \frac{ \lvert\langle B_{\mathbf{p}}\mu_{\theta}, \hat{z}\rangle \rvert }{\|B_{\mathbf{p}}\mu_{\theta}\|_2^2}.
\end{equation}
This yields the practical implementation of Algorithm~\ref{alg:meth1} described in Algorithm \ref{alg:meth1_pract}.

\begin{algorithm2e}[ht]\label{alg:meth1_pract}
\caption{Practical implementation  of naive CL-SGD Algorithm~\ref{alg:meth1}}
\KwData{Sketch $\hat{z}$, $K$ iterations, $\tau$ step size}
 $z_0 \gets \hat{z}$\;
 $\theta_0$ randomly initialized\;
 \For{$k = 0,..., K$}{Generate a grid $\mathbf{p}$ of $P$ points $p_i \sim \mathcal{U}(\mathcal{D})$\;
 $\alpha_{\mathbf{p}} =   \frac{\lvert\langle B_{\mathbf{p}}\mu_{\theta_k}, \hat{z}\rangle \rvert}{\|\langle B_{\mathbf{p}}\mu_{\theta_k}\|_2^2}$\;
 Set $H_1(\theta) = \|B_{\mathbf{p}}\alpha_{\mathbf{p}}\mu_{\theta}-\hat{z}\|$;
 $\theta_{k+1} = \theta_k -  \tau \nabla H_1(\theta)$\;}
\end{algorithm2e}
For Algorithm~\ref{alg:meth2}, we remark that at a given step the important fact is that we have two independent random grids. Hence we just generate one grid at each iteration and use the grid from the previous iteration to generate our descent direction. Morevover the descent direction is easily implemented with automatic differentiation by remarking that
\begin{equation}
D_{\mathbf{p},\mathbf{q}} = \nabla_\theta \langle  B_\mathbf{p}\mu_{\theta}(\mathbf{p}),l_\mathbf{q} \rangle
\end{equation}
where $l_\mathbf{q} = B_\mathbf{q}
 \alpha_p \mu_{\theta}(\mathbf{q}) -\hat{z}$ is fixed. This leads to  Algorithm \ref{alg:meth2_pract} (which is the practical implementation of Algorithm~\ref{alg:meth2}). For this algorithm the previous automatic normalization of the gradient has a tendency to fall in a local minimum (clipping effect due to the fact that $\mu_\theta(p)$ is bounded in $[0,1]$), however manually setting this parameter yields good results.

\begin{algorithm2e}[ht]\label{alg:meth2_pract}
\caption{Practical implementation of Unbiased CL-SGD Algorithm~\ref{alg:meth2}}
\KwData{Sketch $\hat{z}$, $K$ iterations, $\tau$ step size, normalization $\alpha$}
 $z_0 \gets \hat{z}$\;
 $\theta_0$ randomly initialized\;
 Generate a grid $\mathbf{q}$ of $P$ points $q_i \sim \mathcal{U}(\mathcal{D})$\;
 \For{$k = 0,..., K$}{
 Generate a grid $\mathbf{p}$ of $P$ points $p_i \sim \mathcal{U}(\mathcal{D})$\;
  $l_\mathbf{q}=B_\mathbf{q}
 \alpha \mu_{\theta}(\mathbf{q}) -\hat{z}$\;
  Set $H_2(\theta)= \re \langle  B_\mathbf{p}\alpha  \mu_{\theta}(\mathbf{p}),l_\mathbf{q} \rangle$\;
 $\theta_{k+1}= \theta_k -\tau \nabla H_2(\theta)$\;
 $\mathbf{q}=\mathbf{p}$\;}
\end{algorithm2e}

Both algorithms update the DNN with information synthetized from the \emph{whole database at each iteration}. The computational cost of each iteration is similar to the cost of computing the gradient of a typical $\ell^2$ loss with a dataset of size $m$ (size of the sketch instead of size of the dataset). The advantage of this method is that the number of iterations required to converge to the sketch matching problem does not depend on the size of the original database compared to traditional learning with batches where passes (epochs) through the whole database are required.

\subsection{Consistency of CL-SGD with the sketch matching problem}\label{sec:consistency}
In this section, we give the necessary lemmas to link our stochastic descent directions with the original sketch matching problem (summarized by the Lemmas of Section~\ref{sec:pres_clsgd}). The central idea  of our method is that we can generally approximate $S\mu_\theta $ with $B_\mathbf{p}\mu_\theta(\mathbf{p})$, which will translate to the chosen stochastic gradients.

\begin{lemma}\label{lem:consistency}
Consider  $\mathcal{S}$ constructed with frequencies $(\omega_l)_{l=1}^m$.  Let $B_\mathbf{p} \in \mathbb{C}^{m\times P}$ with general term $B_{\mathbf{p},l,i} = \frac{e^{-j \langle\omega_l,p_i \rangle}}{P}$ and $\mu \in \mathcal{M}(\mathcal{D})$. Then
  \begin{equation}
   \esp_\mathbf{p}  \left( B_\mathbf{p}\mu(\mathbf{p})\right)=S\mu.
  \end{equation}
\end{lemma}

\begin{proof}
The expectation yields for $l \in \{1,\ldots,m \}$
\begin{equation}
\begin{split}
 [\esp_\mathbf{p}(B_\mathbf{p}\mu(\mathbf{p}))]_l
  &=   \esp_\mathbf{p}  \sum_{r=1}^P \frac{e^{-j  \langle\omega_l,p_r\rangle} \mu(p_r) }{P}.\\
   \end{split}
\end{equation}

As the $p_r$ are i.i.d, we have

\begin{equation}
\begin{split}
[\esp_\mathbf{p}B_\mathbf{p}\mu(\mathbf{p})]_l &= P   \esp_\mathbf{p}  \frac{e^{-j  \langle\omega_l,p_1\rangle}  \mu(p_1) }{P} \\
&=   \int_{p_1 \in \mathcal{D}}   e^{-j  \langle\omega_l,p_1\rangle} \mu(p_1) d p_1=  [S\mu]_l.
 \end{split}
\end{equation}
\end{proof}
This shows that on average  random discretization of the data domain for the forward sketching operator is consistent  with the original sketch.

To calculate the expectation of our stochastic gradients, we provide the following Lemma which gives the expectation of the discretized cross product between two measures. We write $\langle \mu_1 , \mu_2 \rangle_{L^2(\mathcal{D})} :=   \int_{\mathcal{D}}\mu_1(x)\mu_2(x)dx $ the cross product between two densities $\mu_1$ and $\mu_2$.

\begin{lemma}\label{lem:consistency2}
Consider  $\mathcal{S}$ constructed with frequencies $(\omega_l)_{l=1}^m$.  Let $B_\mathbf{p} \in \mathbb{C}^{m\times P}$ with general term $B_{\mathbf{p},l,i} = \frac{e^{-j \langle\omega_l,p_i \rangle}}{P}$ and  $\mu_1,\mu_2 \in \mathcal{M}(\mathcal{D})$. We have

\begin{equation}
\begin{split}
 \esp_\mathbf{p} \left( \langle  B_\mathbf{p}\mu_1(\mathbf{p}),  B_\mathbf{p}\mu_2(\mathbf{p}) \rangle\right)  &= \frac{m}{P} \langle \mu_1 , \mu_2 \rangle_{L^2(\mathcal{D})}  \\
 &+\frac{P-1}{P}\langle \mathcal{S}\mu_1,  \mathcal{S}\mu_2 \rangle.\\
 \end{split}
\end{equation}

\end{lemma}

\begin{proof}
We have
\begin{equation}
\begin{split}
&\esp_\mathbf{p} \langle  B_\mathbf{p}\mu_1(\mathbf{p}),  B_\mathbf{p}\mu_2(\mathbf{p}) \rangle  \\
&=  \esp_\mathbf{p}  (\mu_2(\mathbf{p})^T B_\mathbf{p} ^* B_\mathbf{p}\mu_1(\mathbf{p}) )\\
&=\frac{1}{P^2} \esp_\mathbf{p} \sum_{t=1}^P \mu_2 (p_t)\sum_{g=1}^m e^{j\langle \omega_g ,p_t\rangle} \sum_{r=1}^P e^{-j  \langle\omega_g,p_r\rangle} \mu_1(p_r)\\
&=\frac{1}{P^2}  \sum_{t=1}^P \sum_{g=1}^m\sum_{r=1}^P  \esp_\mathbf{p} e^{j\langle \omega_g ,p_t-p_r\rangle}\mu_2 (p_t)\mu_1(p_r). \\
 \end{split}
\end{equation}
The diagonal terms in the sum $p_t = p_r$ are
\begin{equation}
\begin{split}
D &= \frac{1}{P^2}  \sum_{t=1}^P \sum_{g=1}^m \esp_\mathbf{p} \mu_2 (p_t)\mu_1(p_t)= \frac{m}{P}   \langle \mu_1 , \mu_2 \rangle_{L^2(\mathcal{D})}. \\
 \end{split}
\end{equation}

The non diagonal terms $p_t \neq p_r$ give (with the fact that the $p_i$ are i.i.d.):

\begin{equation}
\begin{split}
N &=\frac{1}{P^2}  \sum_{t=1}^P \sum_{g=1}^m\sum_{r=1, r\neq t}^P  \esp_\mathbf{p} e^{j\langle \omega_g ,p_t-p_r\rangle}\mu_2 (p_t)\mu_1(p_r) \\
&=\frac{P-1}{P}  \sum_{g=1}^m  \left( \esp_\mathbf{p} e^{j\langle \omega_g ,p_1\rangle}\mu_2 (p_1)\right) \left(  \esp_\mathbf{p}e^{-j  \langle\omega_g,p_1\rangle} \mu_1(p_1)\right) \\
&=\frac{P-1}{P}  \sum_{g=1}^m  (\mathcal{S}\mu_2)_g^*(\mathcal{S}\mu_1)_g  \\
&= \frac{P-1}{P}\langle \mathcal{S}\mu_1,  \mathcal{S}\mu_2 \rangle.
 \end{split}
\end{equation}
\end{proof}
We also calculate the variance of the unbiased estimator of the gradient of $G$ thanks to the following Lemma.

\begin{lemma}\label{lem:variance}

Consider  $\mathcal{S}$ constructed with frequencies $(\omega_l)_{l=1}^m$.  Let $B_\mathbf{p} \in \mathbb{C}^{m\times P}$ with general term $B_{p,l,i} = \frac{e^{-j \langle\omega_l,p_i \rangle}}{P}$ and  $\mu_1,\mu_2 \in \mathcal{M}(\mathcal{D})$. We have

\begin{equation}
\begin{split}
&\esp_{\mathbf{p},\mathbf{q}} \vert \langle  B_\mathbf{p}\mu_1(\mathbf{p}),  B_\mathbf{q}\mu_2(\mathbf{q}) -z \rangle \vert^2 \\
&\quad\quad\quad- \vert\esp_{\mathbf{p},\mathbf{q}}   \langle  B_\mathbf{p}\mu_1(\mathbf{p}),  B_\mathbf{q}\mu_2(\mathbf{q}) -z \rangle \vert^2 \\
&= \frac{1}{P^2} \langle \vert\mu_1\vert^2,\vert\mu_2 \vert^2\rangle_{L^2,\vert\mathcal{S}^* \mathbf{1} \vert^2} +  \frac{1}{P} C(\mu_1,\mu_2,z)
 \end{split}
\end{equation}
where we define $\mathcal{S}^*z: p \to  \sum_g z_ge^{j \langle \omega_g,p \rangle}$, and for a kernel $K$ (a function from $\mathcal{D}$ to $\mathbb{R}$), $\langle \nu_1, \nu_2 \rangle_{L^2(\mathcal{D}),K} := \int_{x,y} \nu_1(x)\nu_2(y)h(x-y) dxdy$. and where
\begin{equation}
\begin{split}
 &C(\mu_1,\mu_2,z)\\
 &:=  \frac{P-1}{P} ( \langle  \vert\mathcal{S}^*\mathcal{S}\mu_1 \vert^2, \vert\mu_2 \vert^2\rangle_{L^2} +   \langle  \vert\mathcal{S}^*\mathcal{S}\mu_2 \vert^2, \vert\mu_1 \vert^2\rangle_{L^2}) \\
&+  \re \langle \vert\mathcal{S}^*z\vert^2 -2\mathcal{S}^*z(\mathcal{S}^*\mathcal{S}\mu_2)^*  , \vert\mu_1 \vert^2\rangle_{L^2}\\
&+ 2\re\left(\langle \mathcal{S}\mu_1,z \rangle \langle \mathcal{S}\mu_1,\mathcal{S}\mu_2 \rangle ^*\right)-  \vert \langle \mathcal{S}\mu_1,z \rangle \vert^2 \\
&- \frac{2P-1}{P} \vert \langle \mathcal{S}\mu_1,\mathcal{S}\mu_2\rangle \vert^2
\end{split}
\end{equation}

\end{lemma}

\begin{proof}

We need to calculate a few terms separately. For $y\in\mathbb{C}^m$, using the fact that
 \[\langle  B_\mathbf{p}\mu_1(\mathbf{p}),  z\rangle  \\
=\sum_{g=1}^m\sum_{t=1}^P   e^{j\langle \omega_g ,p_t\rangle}\mu_1 (p_t)z_g,\] we have
\begin{equation}
\begin{split}
&\esp_{\mathbf{p}}  \langle  B_\mathbf{p}\mu_1(\mathbf{p}),  z\rangle  \langle  B_\mathbf{p}\mu_1(\mathbf{p}),  y\rangle^* \\
&=\frac{1}{P^2}  \sum_{g,t} \sum_{\tilde{g},\tilde{t}}  \esp_{\mathbf{p}} \Big(e^{j\langle \omega_g ,p_t\rangle-j\langle \omega_{\tilde{g}} ,p_{\tilde{t}}\rangle} \mu_1 (p_t) \mu_1 (p_{\tilde{t}})z_g y_{\tilde{g}}^*\Big) \\
&=\frac{1}{P^2}  \sum_{g,\tilde{g}} z_g y_{\tilde{g}}^* \sum_{t,\tilde{t}} \esp_{\mathbf{p}}\Big(  e^{j\langle \omega_g ,p_t\rangle-j\langle \omega_{\tilde{g}} ,p_{\tilde{t}}\rangle}\mu_1 (p_t) \mu_1 (p_{\tilde{t}})\Big). \\
 \end{split}
\end{equation}
As the $p_t$ are i.i.d., we have
\begin{equation}
\begin{split}
 &\sum_{t,\tilde{t}}  \esp_{\mathbf{p}} e^{j\langle \omega_g ,p_t\rangle-j\langle \omega_{\tilde{g}} ,p_{\tilde{t}}\rangle}\mu_1 (p_t) \mu_1 (p_{\tilde{t}}) \quad\\
& \quad \quad\quad \quad=  P  \esp_{\mathbf{p}} [e^{-j\langle \omega_{\tilde{g}}-\omega_g ,p_1\rangle} \vert\mu_1 (p_1)\vert^2]\\
 &\quad \quad\quad \quad+ P(P-1) (\mathcal{S}\mu_1)_g^*(\mathcal{S}\mu_1)_{\tilde{g}}.\\
\end{split}
\end{equation}

We obtain

\begin{equation}\label{eq:proof_var00}
\begin{split}
&\esp_{\mathbf{p}}  \langle  B_\mathbf{p}\mu_1(\mathbf{p}),  z\rangle  \langle  B_\mathbf{p}\mu_1(\mathbf{p}),  y\rangle^*\\
&=  \frac{1}{P^2}  \sum_{g,\tilde{g}} \Big( z_g y_{\tilde{g}}^*  P \esp_{\mathbf{p}} [e^{-j\langle \omega_{\tilde{g}}-\omega_g ,p_{1}\rangle} \vert\mu_1 (p_1)\vert^2] \\
&+ P(P-1) (\mathcal{S}\mu_1)_g^*(\mathcal{S}\mu_1)_{\tilde{g}}z_g y_{\tilde{g}}^*  \Big)\\
&=   \frac{1}{P}\esp_{\mathbf{p}}   [ B_\mathbf{p}^*z  (B_\mathbf{p}^*y)^*] (p_1) \vert^2 \vert\mu_1 (p_1)\vert^2]\\
&+\frac{P-1}{P}   \langle \mathcal{S}\mu_1,z \rangle \langle \mathcal{S}\mu_1,y \rangle ^*\\
&=   \frac{1}{P}\langle \mathcal{S}^*z(\mathcal{S}^*y)^*  , \vert\mu_1 \vert^2\rangle_{L^2}+\frac{P-1}{P}   \langle \mathcal{S}\mu_1,z \rangle  \langle \mathcal{S}\mu_1,y \rangle ^*\\
 \end{split}
\end{equation}
where $\mathcal{S}^*z$ is defined in the hypotheses of the Lemma. For $y= z$ we obtain

\begin{equation}
\begin{split}
&\esp_{\mathbf{p}} \vert \langle  B_\mathbf{p}\mu_1(\mathbf{p}),  z \rangle \vert^2 \\
&=   \frac{1}{P}\langle \vert\mathcal{S}^*z\vert^2  , \vert\mu_1 \vert^2\rangle_{L^2}+\frac{P-1}{P}  \vert \langle \mathcal{S}\mu_1,z \rangle \vert^2\\
 \end{split}
\end{equation}

We now calculate the following expectation:

\begin{equation}
\begin{split}
&\esp_{\mathbf{p},\mathbf{q}} \vert \langle  B_\mathbf{p}\mu_1(\mathbf{p}),  B_\mathbf{q}\mu_2(\mathbf{q}) \rangle \vert^2 \\
&=\frac{1}{P^4}  \esp_{\mathbf{p},\mathbf{q}} \vert\sum_{g=1}^m\sum_{t=1} \sum_{r=1}   e^{j\langle \omega_g ,p_t-q_r\rangle}\mu_1 (p_t)\mu_2(q_r) \vert^2\\
&=\frac{1}{P^4}  \sum_{g,t,r} \sum_{\tilde{g},\tilde{t},\tilde{r}}  \esp_{\mathbf{p},\mathbf{q}}\Big( e^{j\langle \omega_g ,p_t-q_r\rangle}e^{-j\langle \omega_{\tilde{g}} ,p_{\tilde{t}}-q_{\tilde{r}}\rangle}\\
&\quad\; \mu_1 (p_t)\mu_2(q_r) \mu_1 (p_{\tilde{t}})\mu_2(q_{\tilde{r}})\Big)\\
\end{split}
\end{equation}
As $\mathbf{p}$ and $\mathbf{q}$ are i.i.d.,
\begin{equation}\label{eq:proof_var0}
\begin{split}
&\esp_{\mathbf{p},\mathbf{q}} \vert \langle  B_\mathbf{p}\mu_1(\mathbf{p}),  B_\mathbf{q}\mu_2(\mathbf{q}) \rangle \vert^2 \\
&=\frac{1}{P^4}  \sum_{g,t,r} \sum_{\tilde{g},\tilde{t},\tilde{r}} \Big( \esp_{\mathbf{p}} e^{j\langle \omega_g ,p_t\rangle-j\langle \omega_{\tilde{g}} ,p_{\tilde{t}}\rangle} \mu_1 (p_t) \mu_1 (p_{\tilde{t}})\Big) \\
& \quad\;\esp_{\mathbf{q}}\Big(e^{-j\langle \omega_g ,q_r\rangle+j\langle \omega_{\tilde{g}} ,q_{\tilde{r}}\rangle} \mu_2(q_r) \mu_2(q_{\tilde{r}})\Big)\\
&=\frac{1}{P^4}  \sum_{g,\tilde{g}} \sum_{t,\tilde{t}}  \esp_{\mathbf{p}} e^{j\langle \omega_g ,p_t\rangle-j\langle \omega_{\tilde{g}} ,p_{\tilde{t}}\rangle} \mu_1 (p_t) \mu_1 (p_{\tilde{t}}) \\
&  \quad\;\sum_{r,\tilde{r}}\esp_{\mathbf{q}}e^{-j\langle \omega_g ,q_r\rangle+j\langle \omega_{\tilde{g}} ,q_{\tilde{r}}\rangle} \mu_2(q_r) \mu_2(q_{\tilde{r}})\\
&=  \frac{1}{P^4} \sum_{g,\tilde{g}} A_{1,g,\tilde{g}} A_{2,g,\tilde{g}}^* \\
 \end{split}
\end{equation}
where, with the decomposition of the sum into diagonal and off-diagonal terms,
\begin{equation}
\begin{split}
A_{i,g,\tilde{g}} &= \sum_{t,\tilde{t}}  \esp_{\mathbf{p}} e^{j\langle \omega_g ,p_t\rangle-j\langle \omega_{\tilde{g}} ,p_{\tilde{t}}\rangle} \mu_i (p_t) \mu_i (p_{\tilde{t}})\\
 &= P  \esp_{\mathbf{p}} e^{-j\langle \omega_{\tilde{g}}-\omega_g ,p_{t}\rangle} \vert\mu_i (p_t)\vert^2\\
 &\quad\;+ P(P-1) (\mathcal{S}\mu_i)_g^*(\mathcal{S}\mu_i)_{\tilde{g}}.\\
\end{split}
\end{equation}

We obtain

\begin{equation}\label{eq:proof_var1}
\begin{split}
&P^2\esp_{\mathbf{p},\mathbf{q}} \vert \langle  B_\mathbf{p}\mu_1(\mathbf{p}),  B_\mathbf{q}\mu_2(\mathbf{q}) \rangle \vert^2 \\
&=    \sum_{g,\tilde{g}} ( \esp_{\mathbf{p}} e^{-j\langle \omega_{\tilde{g}}-\omega_g ,p_{t}\rangle}  \vert\mu_1 (p_t)\vert^2  \esp_{\mathbf{q}} e^{j\langle \omega_{\tilde{g}}-\omega_g ,q_{r}\rangle} \vert\mu_2(q_r)\vert^2\\
& +  (P-1) (\mathcal{S}\mu_1)_g^*(\mathcal{S}\mu_1)_{\tilde{g}} \esp_{\mathbf{q}}  e^{j\langle \omega_{\tilde{g}}-\omega_g ,q_{r}\rangle} \vert\mu_2(q_r)\vert^2 \\
&+ (P-1) (\mathcal{S}\mu_2)_g(\mathcal{S}\mu_2)_{\tilde{g}}^* \esp_{\mathbf{p}}  e^{-j\langle \omega_{\tilde{g}}-\omega_g ,p_{t}\rangle}  \vert\mu_1 (p_t)\vert^2 \\
&+(P-1)^2 (\mathcal{S}\mu_1)_g^*(\mathcal{S}\mu_1)_{\tilde{g}} (\mathcal{S}\mu_2)_g(\mathcal{S}\mu_2)_{\tilde{g}}^*
 \end{split}
\end{equation}

with

\begin{equation}
\begin{split}
& \sum_{g,\tilde{g}} ( \esp_{\mathbf{p}} e^{-j\langle \omega_{\tilde{g}}-\omega_g ,p_{t}\rangle}  \vert\mu_1 (p_t)\vert^2  \esp_{\mathbf{q}} e^{j\langle \omega_{\tilde{g}}-\omega_g ,q_{r}\rangle} \vert\mu_2(q_r)\vert^2\\
&  =\esp_{\mathbf{p}} \esp_{\mathbf{q}}  \left(\sum_{g,\tilde{g}}e^{-j\langle \omega_{\tilde{g}}-\omega_g ,p_{t}\rangle}    e^{j\langle \omega_{\tilde{g}}-\omega_g ,q_{r}\rangle} \right) \vert\mu_1 (p_t)\vert^2\vert\mu_2(q_r)\vert^2\\
 \end{split}
\end{equation}

We have inside the expectation,
\begin{equation}
\begin{split}
&   \left(\sum_{g,\tilde{g}}e^{j\langle \omega_{\tilde{g}},q_r-p_t\rangle}    e^{-j\langle \omega_g ,q_{r}-p_t\rangle} \right) \vert\mu_1 (p_t)\vert^2\vert\mu_2(q_r)\vert^2\\
&  = \left(\sum_{\tilde{g}}e^{j\langle \omega_{\tilde{g}},q_r-p_t\rangle}    \sum_{g}e^{-j\langle \omega_g ,q_{r}-p_t\rangle} \right) \vert\mu_1 (p_t)\vert^2\vert\mu_2(q_r)\vert^2\\
&  = \left(\sum_{\tilde{g}}e^{j\langle \omega_{\tilde{g}},q_r-p_t\rangle}    \sum_{g}e^{-j\langle \omega_g ,q_{r}-p_t\rangle} \right) \vert\mu_1 (p_t)\vert^2\vert\mu_2(q_r)\vert^2\\
\end{split}
\end{equation}
This gives
\begin{equation}\label{eq:proof_var2}
\begin{split}
&P^2\esp_{\mathbf{p},\mathbf{q}} \vert \langle  B_\mathbf{p}\mu_1(\mathbf{p}),  B_\mathbf{q}\mu_2(\mathbf{q}) \rangle \vert^2\\
& = \esp_{\mathbf{p},\mathbf{q}}  \vert[\mathcal{S}^* \mathbf{1}] (q_r-p_t)\vert^2 \vert\mu_1(p_t)\vert^2\vert\mu_2(q_r) \vert^2\\
& = \langle \vert\mu_1\vert^2,\vert\mu_2 \vert^2\rangle_{L^2(\mathcal{D}),\vert\mathcal{S}^* \mathbf{1} \vert^2}.
 \end{split}
\end{equation}
where we define for a kernel $K$ (a function from $\mathcal{D}$ to $\mathbb{R}$), $\langle \nu_1, \nu_2 \rangle_{L^2(\mathcal{D}),K} = \int_{x,y} \nu_1(x)\nu_2(y)h(x-y) dxdy$.

We calculate the second term (and similarly the third term) of the right hand side of~\eqref{eq:proof_var1}.
\begin{equation}\label{eq:proof_var3}
\begin{split}
  &\sum_{g,\tilde{g}} (\mathcal{S}\mu_1)_g^*(\mathcal{S}\mu_1)_{\tilde{g}} \esp_{\mathbf{q}}  e^{j\langle \omega_{\tilde{g}}-\omega_g ,q_{r}\rangle} \vert\mu_2(q_r)\vert^2\\
  &\esp_{\mathbf{q}} \sum_{g,\tilde{g}} e^{-j\langle\omega_g ,q_{r}\rangle}(\mathcal{S}\mu_1)_g^*e^{j\langle \omega_{\tilde{g}} ,q_{r}\rangle}(\mathcal{S}\mu_1)_{\tilde{g}}   \vert\mu_2(q_r)\vert^2\\
  &=  \esp_{\mathbf{q}}  \left( \vert\mathcal{S}^*\mathcal{S}\mu_1 \vert^2_{q_r}\vert\mu_2(q_r)\vert^2\right)= \langle  \vert\mathcal{S}^*\mathcal{S}\mu_1 \vert^2, \vert\mu_2 \vert^2\rangle_{L^2(\mathcal{D})}.
  \end{split}
\end{equation}

The fourth term of the right hand side of~\eqref{eq:proof_var1} yields

\begin{equation}\label{eq:proof_var4}
\begin{split}
 &\sum_{g,\tilde{g}} (\mathcal{S}\mu_1)_g^*(\mathcal{S}\mu_1)_{\tilde{g}} (\mathcal{S}\mu_2)_g(\mathcal{S}\mu_2)_{\tilde{g}}^*  \\
 &=  \sum_g  (\mathcal{S}\mu_1)_g^* (\mathcal{S}\mu_2)_g\langle \mathcal{S}\mu_1, \mathcal{S}\mu_2\rangle = \vert\langle \mathcal{S}\mu_1, \mathcal{S}\mu_2\rangle \vert^2.\\
 \end{split}
\end{equation}

Going back to~\eqref{eq:proof_var0}, we have using the expressions~\eqref{eq:proof_var2},~\eqref{eq:proof_var3} and~\eqref{eq:proof_var4} in~\eqref{eq:proof_var1}

\begin{equation}\label{eq:proof_var5}
\begin{split}
&\esp_{\mathbf{p},\mathbf{q}} \vert \langle  B_\mathbf{p}\mu_1(\mathbf{p}),  B_\mathbf{q}\mu_2(\mathbf{q}) \rangle \vert^2\\
&=  \frac{1}{P^2} \langle \vert\mu_1\vert^2,\vert\mu_2 \vert^2\rangle_{L^2(\mathcal{D}),\vert\mathcal{S}^* \mathbf{1} \vert^2}\\
& +  \frac{P-1}{P^2} \langle  \vert\mathcal{S}^*\mathcal{S}\mu_1 \vert^2, \vert\mu_2 \vert^2\rangle_{L^2(\mathcal{D})} \\
&+  \frac{P-1}{P^2} \langle  \vert\mathcal{S}^*\mathcal{S}\mu_2 \vert^2, \vert\mu_1 \vert^2\rangle_{L^2(\mathcal{D})} \\
&+ \frac{(P-1)^2}{P^2} \vert\langle \mathcal{S}\mu_1, \mathcal{S}\mu_2\rangle \vert^2.
 \end{split}
\end{equation}

By developing expressions, we have

\begin{equation}
\begin{split}
&\esp_{\mathbf{p},\mathbf{q}} \vert \langle  B_\mathbf{p}\mu_1(\mathbf{p}),  B_\mathbf{q}\mu_2(\mathbf{q}) -z \rangle  \\
&\quad\quad\quad- \esp_{\mathbf{p},\mathbf{q}}   \langle  B_\mathbf{p}\mu_1(\mathbf{p}),  B_\mathbf{q}\mu_2(\mathbf{q}) -z \rangle \vert^2 \\
& = \esp_{\mathbf{p},\mathbf{q}} \vert \langle  B_\mathbf{p}\mu_1(\mathbf{p}),  B_\mathbf{q}\mu_2(\mathbf{q}) -z \rangle   \vert^2- \vert \langle\mathcal{S}\mu_1, \mathcal{S} \mu_2 - z \rangle \vert^2 \\
&=\esp_{\mathbf{p},\mathbf{q}} \vert \langle  B_\mathbf{p}\mu_1(\mathbf{p}),  B_\mathbf{q}\mu_2(\mathbf{q})\rangle\vert^2 \\
&- 2 \esp_{\mathbf{p},\mathbf{q}} \re\left( \langle  B_\mathbf{p}\mu_1(\mathbf{p}),  z \rangle \langle  B_\mathbf{p}\mu_1(\mathbf{p}),  B_\mathbf{q}\mu_2(\mathbf{q})\rangle^*\right)  \\
&+\esp_{\mathbf{p},\mathbf{q}} \vert\langle  B_\mathbf{p}\mu_1(\mathbf{p}),  z \rangle   \vert^2- \vert \langle\mathcal{S}\mu_1, \mathcal{S} \mu_2 - z \rangle \vert^2 \\
&=\esp_{\mathbf{p},\mathbf{q}} \vert \langle  B_\mathbf{p}\mu_1(\mathbf{p}),  B_\mathbf{q}\mu_2(\mathbf{q})\rangle\vert^2 \\
&- 2 \esp_{\mathbf{p},\mathbf{q}} \re\left( \langle  B_\mathbf{p}\mu_1(\mathbf{p}),  z \rangle \langle  B_\mathbf{p}\mu_1(\mathbf{p}),  \mathcal{S}\mu_2\rangle^*\right)  \\
&+\esp_{\mathbf{p},\mathbf{q}} \vert\langle  B_\mathbf{p}\mu_1(\mathbf{p}),  z \rangle   \vert^2- \vert \langle\mathcal{S}\mu_1, \mathcal{S} \mu_2 - z \rangle \vert^2 \\
\end{split}
\end{equation}

Using equation~\eqref{eq:proof_var5} and the fact that $ \esp_{\mathbf{p},\mathbf{q}} \re \langle  B_\mathbf{p}\mu_1(\mathbf{p}),  z \rangle=  \re \langle  \esp_{\mathbf{p},\mathbf{q}} B_\mathbf{p}\mu_1(\mathbf{p}),  z \rangle = \re \langle  \mathcal{S}\mu_1,  z \rangle$ with Lemma~\ref{lem:consistency}, equation~\eqref{eq:proof_var00}, we have

\begin{equation}
\begin{split}
&\esp_{\mathbf{p},\mathbf{q}} \vert \langle  B_\mathbf{p}\mu_1(\mathbf{p}),  B_\mathbf{q}\mu_2(\mathbf{q}) -z \rangle  \\
&\quad\quad\quad- \esp_{\mathbf{p},\mathbf{q}}   \langle  B_\mathbf{p}\mu_1(\mathbf{p}),  B_\mathbf{q}\mu_2(\mathbf{q}) -z \rangle \vert^2 \\
&= \frac{1}{P^2} \langle \vert\mu_1\vert^2,\vert\mu_2 \vert^2\rangle_{L^2(\mathcal{D}),\vert\mathcal{S}^* \mathbf{1} \vert^2}\\
& +  \frac{P-1}{P^2} \langle  \vert\mathcal{S}^*\mathcal{S}\mu_1 \vert^2, \vert\mu_2 \vert^2\rangle_{L^2(\mathcal{D})} \\
&+  \frac{P-1}{P^2} \langle  \vert\mathcal{S}^*\mathcal{S}\mu_2 \vert^2, \vert\mu_1 \vert^2\rangle_{L^2(\mathcal{D})} \\
&+ \frac{(P-1)^2}{P^2} \vert\langle \mathcal{S}\mu_1, \mathcal{S}\mu_2\rangle \vert^2\\
&-2\frac{1}{P} \re \langle (\mathcal{S}^*z)(\mathcal{S}^*\mathcal{S}\mu_2)^*  , \vert\mu_1 \vert^2\rangle_{L^2}\\
&-2 \frac{P-1}{P} \re  \left(\langle \mathcal{S}\mu_1,z \rangle  \langle \mathcal{S}\mu_1,\mathcal{S}\mu_2 \rangle ^*\right)\\
& + \frac{1}{P}\langle \vert\mathcal{S}^*z\vert^2  , \vert\mu_1 \vert^2\rangle_{L^2(\mathcal{D})}\\
&+\frac{P-1}{P}  \vert \langle \mathcal{S}\mu_1,z \rangle \vert^2\\
&-  \vert \langle \mathcal{S}\mu_1,z \rangle \vert^2 \\
&+ 2 \re\left(\langle \mathcal{S}\mu_1,z \rangle\langle \mathcal{S}\mu_1,\mathcal{S}\mu_2 \rangle^*\right)\\
&-  \vert \langle \mathcal{S}\mu_1,\mathcal{S}\mu_2\rangle \vert^2
 \end{split}
\end{equation}
Regrouping terms yields

\begin{equation}
\begin{split}
&\esp_{\mathbf{p},\mathbf{q}} \vert \langle  B_\mathbf{p}\mu_1(\mathbf{p}),  B_\mathbf{q}\mu_2(\mathbf{q}) -z \rangle  \\
&\quad\quad\quad- \esp_{\mathbf{p},\mathbf{q}}   \langle  B_\mathbf{p}\mu_1(\mathbf{p}),  B_\mathbf{q}\mu_2(\mathbf{q}) -z \rangle \vert^2 \\
&= \frac{1}{P^2} \langle \vert\mu_1\vert^2,\vert\mu_2 \vert^2\rangle_{L^2,\vert\mathcal{S}^* \mathbf{1} \vert^2}\\
& +  \frac{P-1}{P^2} \left(\langle  \vert\mathcal{S}^*\mathcal{S}\mu_1 \vert^2, \vert\mu_2 \vert^2\rangle_{L^2} + \langle  \vert\mathcal{S}^*\mathcal{S}\mu_2 \vert^2, \vert\mu_1 \vert^2\rangle_{L^2} \right)\\
&+ \frac{1}{P} \re \langle \vert\mathcal{S}^*z\vert^2 -2\mathcal{S}^*z(\mathcal{S}^*\mathcal{S}\mu_2)^*  , \vert\mu_1 \vert^2\rangle_{L^2}\\
&+ \frac{2}{P} \re\left(\langle \mathcal{S}\mu_1,z \rangle \langle \mathcal{S}\mu_1,\mathcal{S}\mu_2 \rangle ^*\right)\\
&-  \frac{1}{P}\vert \langle \mathcal{S}\mu_1,z \rangle \vert^2 \\
&- \frac{2P-1}{P^2} \vert \langle \mathcal{S}\mu_1,\mathcal{S}\mu_2\rangle \vert^2
 \end{split}
\end{equation}
\end{proof}
We have that the variance converges to $0$ at
the typical rate $1/P$.
%
%
\subsection{Convergence analysis}
The advantage of  Lemma~\ref{lem:gradient_meth2} showing the consistency in expectation of our CL-SGD update with the gradient of the ideal sketch matching problem is that we can use out of the box results of the convergence of SGD. Under some hypotheses on the stochastic descent direction, it is possible de to show that
\begin{equation}\label{eq:limit_gradient}
  \lim_{k \to \infty} \inf \esp \|\nabla G( \theta_k )\|_2^2 = 0
\end{equation}
i.e., in expectation, our estimate converges to a critical point of $G$ (a point $\theta$ such that $\nabla G(\theta) = 0$) which is the best we can hope for in this non-convex setting (with minimal hypotheses).

Specifically, we can use results in the non-convex setting from the review article \cite[Theorem 4.9]{bottou2018optimization}.  We suppose that $G$ has Lipschitz gradient, $G$ is lower bounded, that the expectation of our descent direction is exactly the gradient of $G$ and the variance is bounded by $M_1 + M_2 \|\nabla G(\theta)\|_2^2$ (which is verified with $M_2 =0$ as long as the $\mu_\theta$ are bounded in $L^2$ norm thanks to Lemma~\ref{lem:variance}). Then with a  sequence of diminishing step sizes  $\tau_k>0$ such that $\| \sum_{k} \tau_k\|= \infty$ and  $\| \sum_{k} \tau_k^2\| < \infty$ (tipically one may choose $\tau_k=\frac{1}{k}$), we have that the equation \eqref{eq:limit_gradient} is verified. With fixed step sizes the convergence results include the variance.

\section{Experimental results}\label{sec: 5}

\subsection{Synthetic data}
We first test our proposed approaches with 2-D synthetic data. The used training dataset is made of $n = 10^6$ samples which are generated from a spiral with a radius of circular curve from 0.3 to 1 and spiral length $2\pi$ (shown as (a) in Fig. \ref{fig_train_data}). The ReLU network $f_\theta$ with 3 hidden layers, each layer contains 64, 64, 128 neurons respectively. The dataset is compressed into a sketch of size $m=500$, i.e. with a compression ratio $r=2000$.

Figure \ref{fig_train_data} (b) shows the prior model learned using the initial method, i.e. the method from \cite{shi2023compressive} described in equation \eqref{eq_init_method}, with a sketch of size $m=500$, $P=900$ points uniformly generated on the grid of the data domain.
Figure \ref{fig_train_data} (c) and (d) shows the prior model learned with Algorithm~\ref{alg:meth1_pract} and~\ref{alg:meth2_pract}. We use $P = 1600$ points randomly generated on the data domain for algorithm ~\ref{alg:meth1_pract}. For  algorithm~\ref{alg:meth2_pract}, we use $P = 1000$ points and fix the value of $\alpha = 50$.
The proposed algorithms are capable of
recovering good approximations of the probability distribution of sample data while taking about half the training time (about 6 minutes) of the initial method.

\begin{figure}[h]
\begin{tabular}{cc}
\includegraphics[width=0.4\linewidth]{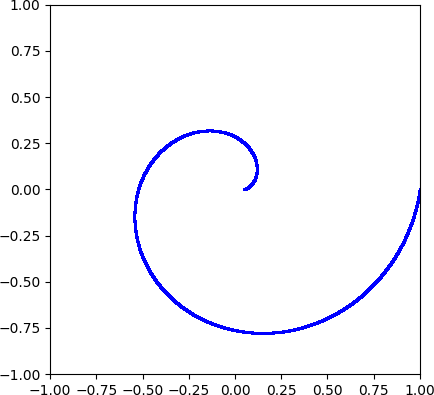}
&\includegraphics[width=0.45\linewidth]{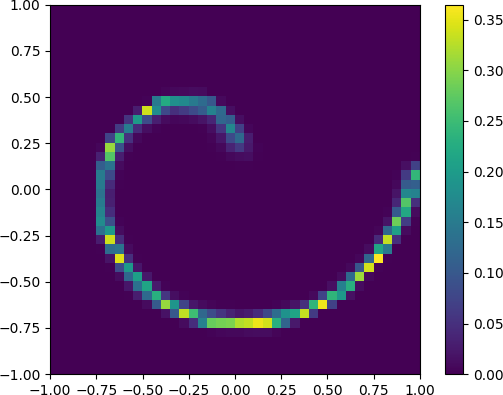} \\
\footnotesize{(a)} & \footnotesize{(b)}\\
\includegraphics[width=.45\linewidth]{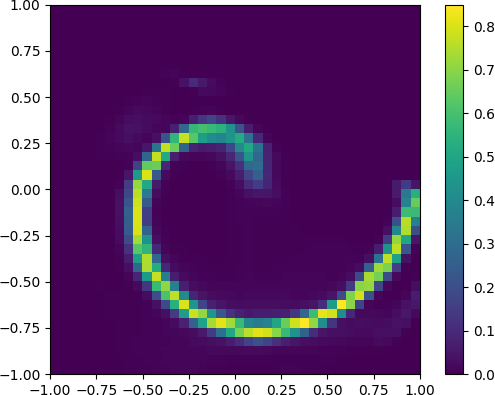}
&\includegraphics[width=.45\linewidth]{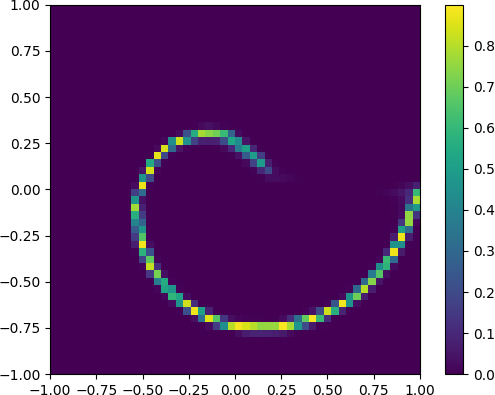}\\
\footnotesize{(c)} & \footnotesize{(d)}\\
\end{tabular}
\caption{Results for the learning of densities. (a) The training data. The prior model learned with (b) regular discretization (previous work \cite{shi2023compressive}), (c) Naïve CL-SGD Algorithm~\ref{alg:meth1_pract}  and (d) Unbiased SL-SGD Algorithm~\ref{alg:meth2_pract}.}
\label{fig_train_data}
\end{figure}

\paragraph{Denoising results}
Similar to \cite{shi2023compressive}, we apply the learned regularization term to solve the variational problem defined in equation \eqref{eq:denoising}. This results in the minimization of:
\begin{equation}
\begin{split}
G(u) 
&=\|u-v\|_2^2 + \lambda \| f_\theta (u) \| _2^2.
\end{split}
\end{equation}
Furthermore, we can conveniently compute the gradient using automatic differentiation. It is important to mention that this denoising approach can be seamlessly extended to address diverse linear inverse problems (beyond the scope of this article). It is  indeed now well known with plug and play approach that denoisers capture enough  information on the data for the solving of inverse problems.

The effectiveness of the learned regularizers is evaluated within the context of denoising white Gaussian noise. Specifically, the noisy dataset consists of 500 samples generated with a noise level of $\sigma^2=0.15$. We manually select the optimal hyperparameter values, including the gradient step size and the regularization parameter $\lambda$, for each individual model.

Figure \ref{fig_denoise_2d} visually illustrates the 2-D denoising results with different noise levels $\sigma^2 = 0.15, 0.2$. From top to bottom, the figure shows the denoising results using regularizers learned from the compressed dataset 2000 times smaller with the initial method \cite{shi2023compressive},  Algorithm~\ref{alg:meth1_pract} and~\ref{alg:meth2_pract} respectively. Table \ref{tab_snr} shows the average gain on SNR (Signal to Noise Ratio) of denoising results using different models.  We observe unbiased CL-SGD yields the best performance. However it must be noted that its parametrization is much harder than Naïve SGD  in practive.

\begin{figure}[h]
\begin{tabular}{cc}
\includegraphics[width=.45\linewidth]{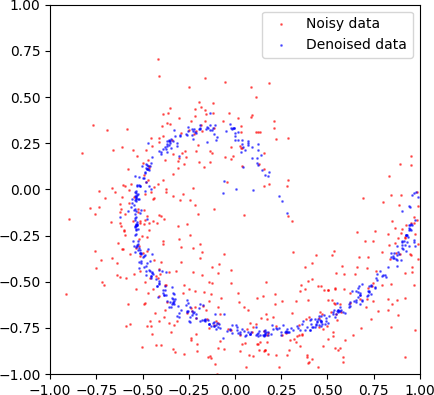} &\includegraphics[width=.45\linewidth]{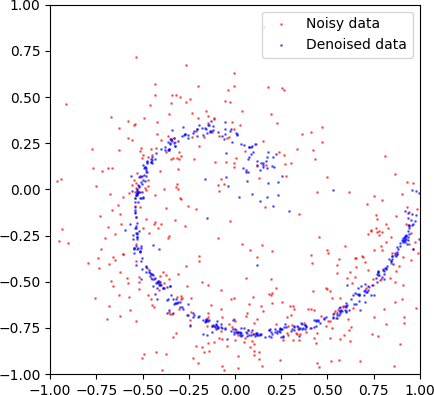}\\
\includegraphics[width=.45\linewidth]{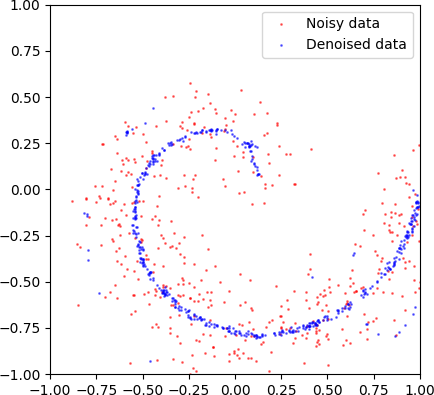}
&\includegraphics[width=.45\linewidth]{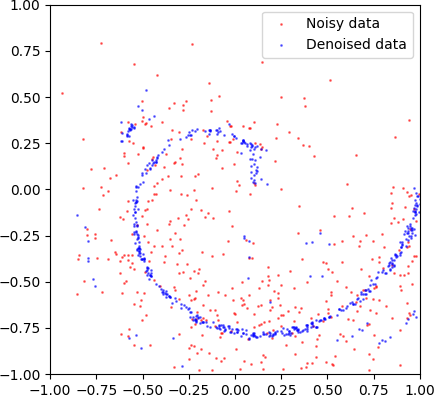}\\
\includegraphics[width=.45\linewidth]{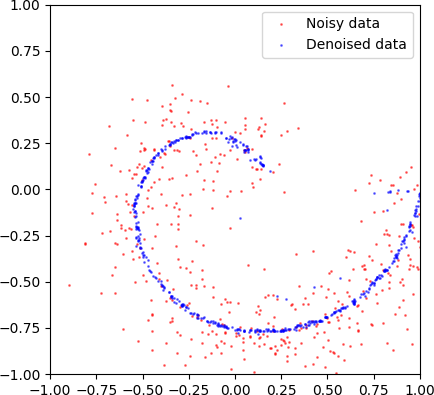} 
&\includegraphics[width=.45\linewidth]{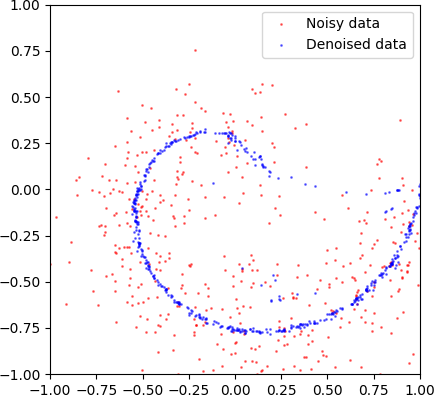}
\end{tabular}
\caption{Denoising results with noise level $\sigma = 0.15$ (left) and $\sigma = 0.2$ (right). Regularizers learned from 2000 times compressed dataset ($m=500$) with initial method  \cite{shi2023compressive}  (1st row), Algorithm~\ref{alg:meth1_pract} (2nd row) and Algorithm~\ref{alg:meth2_pract} (3rd row). }
\label{fig_denoise_2d}
\end{figure}

\begin{table}[h]
\centering
\caption{Average gain on SNR}
\label{tab_snr}
\begin{tabular}{|c|c|c|c|}
 \hline
Gain SNR& Inital method & Algo 3 & Algo 4\\
 \hline
 $\sigma^2 = 0.15$ & +2.065 & +1.954 & +2.89\\
  \hline
 $\sigma^2 = 0.2$ & +1.613 & +1.583& +2.45\\
  \hline
\end{tabular}
\end{table}

\subsection{Audio denoising}
To illustrate the advantages of the proposed methods in the same setting as \cite{shi2023compressive} on real data (as we were unable to deal with dimensions larger than 3 with the regular discretization of $\mathcal{S}$), we perform experiments on recorded musical notes (monophonic 16kHz audio snippets) from the NSynth dataset \cite{nsynth2017}.
To compare, we use the same compressed dataset as in the previous work. That is, the training dataset comprises 0.125 seconds of audio extracted from an acoustic guitar.
After filtering the normalized audio data $s$ by two 4th-order Butterworth low-pass filters $h_1$  and $h_2$ with a cutoff frequency of 1.5kHz and 3.75kHz, three frequency responses are constructed with $s_1 = h_1 * s$, $s_2 = h2 * (s - s_1)$, and $s_3 = s - s1- s2$. 
Then the frequency responses are concatenated, hence the training set is of dimension $2000 \times 3$; i.e. 2000 samples in dimension 3. 
The regularizer is learned from a sketch of size $m=200$, $i.e.$ the dataset is compressed by a factor of $30$.
Once the regularizer is learned, it is applied to denoise the audio that has been corrupted by  Gaussian white noise at noise level $\sigma^2 = 0.1$.

Figure \ref{fig_denoise_audio} and Figure~\ref{fig_denoise_audio2} demonstrate the audio denoising results. We gain 1.49dB on SNR with Algorithm~\ref{alg:meth1_pract} and 1.97dB with Algorithm~\ref{alg:meth2_pract} in the case of small noise ($\sigma^2 = 0.1$). Worse denoising results (gain 1.36dB) are achieved by using the initial approach \eqref{eq_init_method} in the previous work \cite{shi2023compressive}.
Compared to the initial method, which use $P=8000$ points to discretize the sketching operator, the results displayed in the figure is achieved with a model learned with $P = 4000$ points. Consequently, the new proposed algorithms exhibits enhanced efficiency (4 times faster) compared to the initial approach.

\begin{figure}[h]
\includegraphics[width=.45\linewidth]{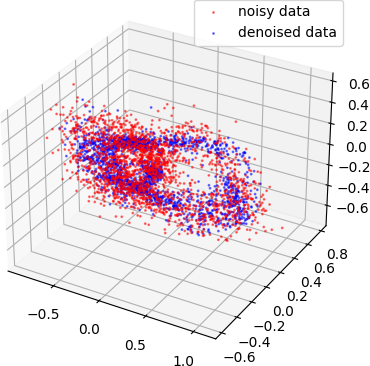} 
\includegraphics[width=.45\linewidth]{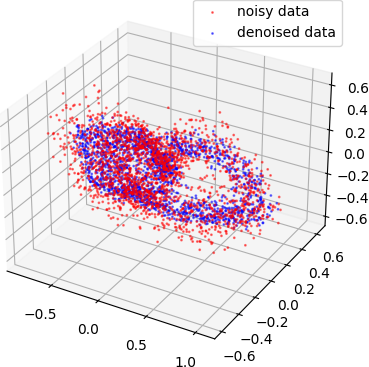}
\caption{Audio denoising results with regularizers learned from a 30 times compressed dataset with Algorithm 3 (left) and 4 (right). }
\label{fig_denoise_audio}
\end{figure}

\begin{figure}[h]
\includegraphics[width=1\linewidth]{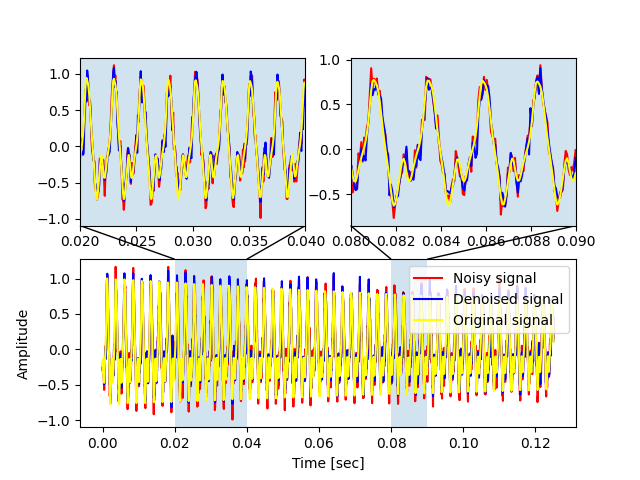} 
\includegraphics[width=1\linewidth]{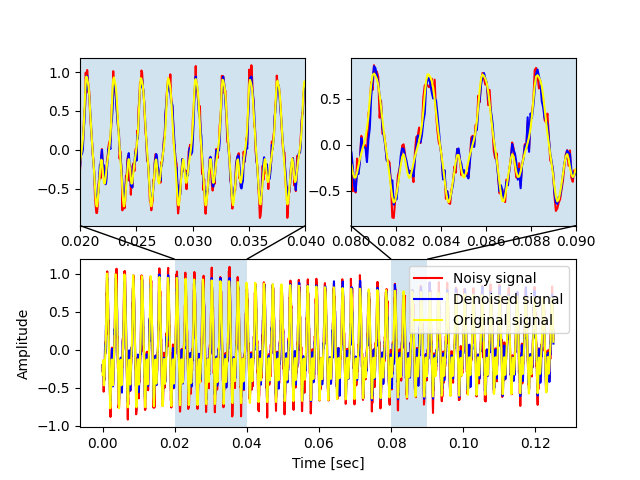}
\caption{Audio results with regularizers learned from a 30 times compressed dataset with Algorithm 3 (top) and 4 (bottom). }
\label{fig_denoise_audio2}
\end{figure}

\subsection{Image denoising results}
We show that our proposed method can be applied in the higher dimensional context of patch-based image denoising. The initial training data comprises $n=4\times 10^6$ patches of size $3 \times 3$ (i.e. $d=9$). These patches are compressed into a sketch of size $m = 10^4$. Figure \ref{fig_denoise_cameraman} shows test images (1st row) and their noisy version (2nd row) used in the experiment (with a noise level of $\sigma=0.07$.

It also shows the denoising result achieved by using a regularizer learned via Algorithm~\ref{alg:meth1_pract} (4th row). The regularizer is trained using a ReLU network consisting of 5 hidden layers, with respectively 64, 64, 128, 196, 196 neurons in each layer. We use $P=8 \times 10^4$ random points on the data domain and the Adam optimizer with a learning rate of $10^{-3} $. The learning process takes 5.7 hours on a machine with 2 * AMD EPYC 7452 32-Core Processor and 256 GB RAM. 

We compare the denoising results with that achieved with a GMM regularizer learned from the same sketch using the method from~\cite{shi2022sketching} (shown in the 3rd row in Figure \ref{fig_denoise_cameraman}). Note that only one iteration of the denoising method is used with an optimal choice of regularization parameter for all methods. Hence we observe directly the denoising effect of the two different regularizers. This comparison demonstrates that Algorithm~\ref{alg:meth1_pract} yields similar result as the GMM method, thus showing  that CL-SGD can be used successfully for the learning of a deep prior from a sketch in high dimension. We attribute the fact that the denoising performance  is not increased to the limited denoising possibilites on independent $3\times 3$ patches. Further work on accelerating our algorithms to manage bigger patches should show the improved representation capabilities of DNN.

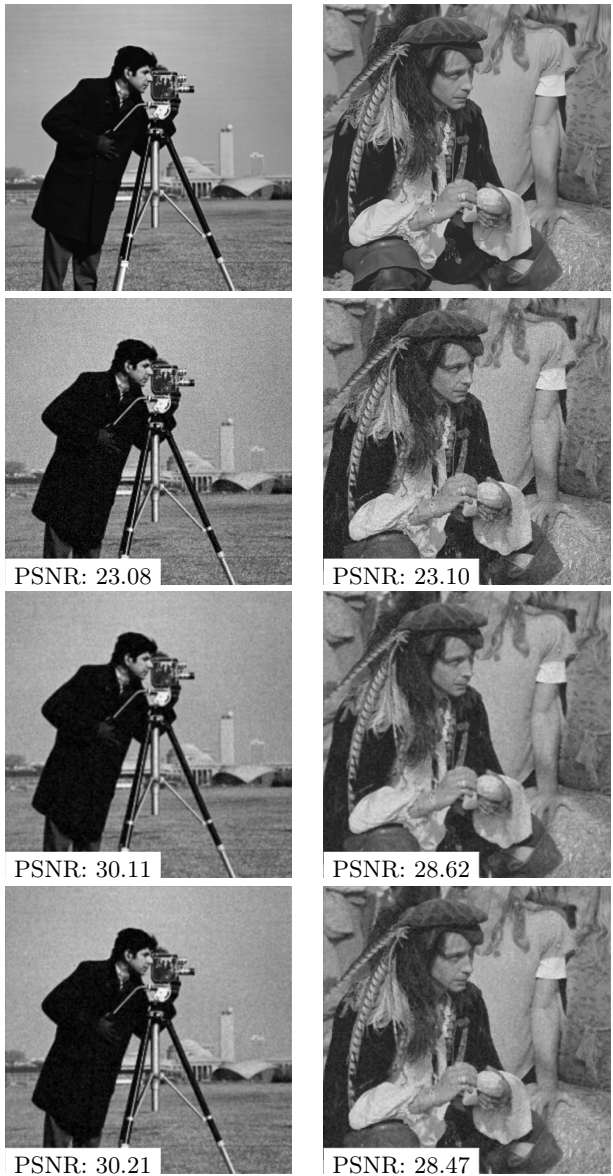
\begin{figure}
\begin{tabular}{cc}
\includegraphics[width=0.5\linewidth]{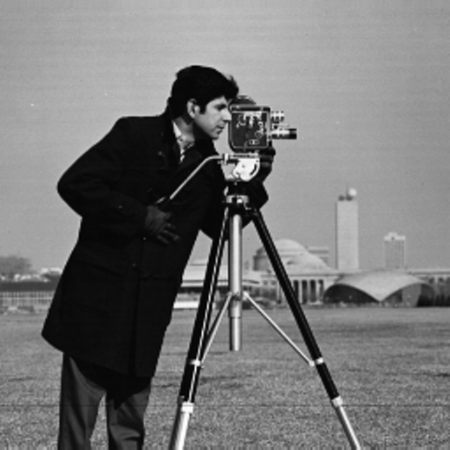}
&\includegraphics[width=0.5\linewidth]{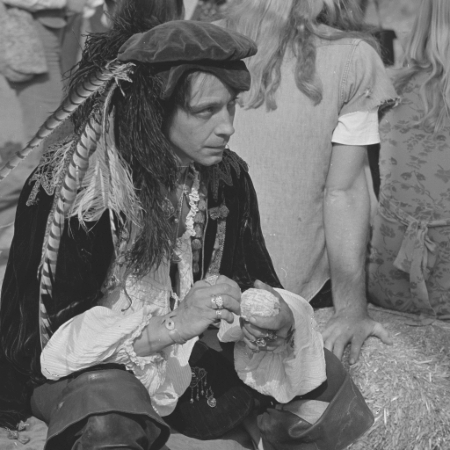}\\
\includegraphics[width=0.5\linewidth]
<\put (0,0){\fcolorbox{white}{white}{\textcolor{black}{\scriptsize{PSNR: 23.08}}}}>{noisy_cameraman.png} &\includegraphics[width=0.5\linewidth]
<\put (0,0){\fcolorbox{white}{white}{\textcolor{black}{\scriptsize{PSNR: 23.10}}}}>{noisy_pirate_c.png}\\
\includegraphics[width=0.5\linewidth]
<\put (0,0){\fcolorbox{white}{white}{\textcolor{black}{\scriptsize{PSNR: 30.11}}}}>{epll_cameraman_c.png}
&\includegraphics[width=0.5\linewidth]
<\put (0,0){\fcolorbox{white}{white}{\textcolor{black}{\scriptsize{PSNR: 28.62}}}}>{epll_pirate_c.png}\\
\includegraphics[width=0.5\linewidth]
<\put (0,0){\fcolorbox{white}{white}{\textcolor{black}{\scriptsize{PSNR: 30.21}}}}>{de_cameraman_m1_c.png}
&\includegraphics[width=0.5\linewidth]
<\put (0,0){\fcolorbox{white}{white}{\textcolor{black}{\scriptsize{PSNR: 28.47}}}}>{de_pirate_m1_c.png}\\
\end{tabular}
\caption{Original test images (1st row). Noisy test images (2nd row) with noise level $\sigma= 0.07$. 
Denoised images with 10-GMM learned from the sketch (3rd row).
Denoised images with regularizers learned with Naive CL-SGD (Algorithm~\ref{alg:meth1_pract}) (4th row).}
\label{fig_denoise_cameraman}
\end{figure}

We present denoising results with Unbiased CL-SGD  in Figure~\ref{fig_denoise_m2}. While Algorithm~\ref{alg:meth2_pract} performs a strong denoising visually, we observe a lack of preservation of the contrast of the original image (hence the PSNR metric is not improved by this method). Unfortunately, the difficulty of parametrization that we observe in the 2D synthetic context becomes even more harder in dimension 9 ($3\times3$ patches). This makes the study of acceleration method fo unbiased CL-SGD even more important for an easy use in the context of image denoising. Moreover forcing some knowledge on the regularizer (e.g. $R_\theta(u) = 0$ if $u$ is a constant patch) might facilitate its training process.

\begin{figure}
\begin{tabular}{cc}
\includegraphics[width=0.5\linewidth]{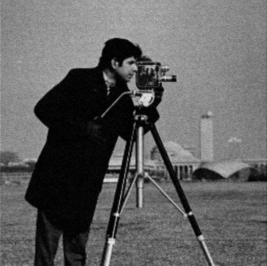}
&\includegraphics[width=0.5\linewidth]{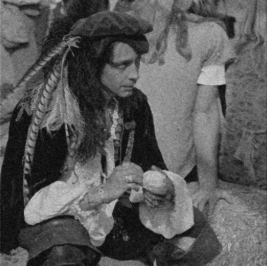}\\
\end{tabular}
\caption{Denoised images with regularizers learned with Unbiased CL-SGD (Algorithm~\ref{alg:meth2_pract}).}
\label{fig_denoise_m2}
\end{figure}

\section{Conclusions}\label{sec: 6}
In this work, we adapt the sketching framework to the learning of a regularizer parametrized by a DNN. This is achieved with a new stochastic gradient descent algorithm, CL-SGD, with convergence guarantees where randomness is used for the discretization of the sketching operator. Our method outperforms the previous work based on  the  deterministic discretization of the forward operator $\mathcal{S}$ on a regular grid. Moreover, it must be noted that only a database of clean images suffices to learn the deep prior.
Experimental results  obtained from  synthetic 2-D data, audio denoising (data in 3D), and real images validates CL-SGD both in terms of computational efficiency and quality of the trained regularizer.

Many questions arise from this work. While the design of the sketching operator and its theoretical justification was heavily influenced by the prior model (e.g. GMM), this design should be revisited and  generalized to DNN based priors. From a practical perspective, it would be interesting to compare denoising results with our compressive approach on huge image databases with plug and play approaches where a training process is performed on pairs of noisy/clean images. Also acceleration methods (with inertia) should be investigated to improve the learning time of the deep prior.

\bibliographystyle{splncs04}
\bibliography{refs}

\begin{thebibliography}{10}
\providecommand{\url}[1]{\texttt{#1}}
\providecommand{\urlprefix}{URL }
\providecommand{\doi}[1]{https://doi.org/#1}

\bibitem{bottou2018optimization}
Bottou, L., Curtis, F.E., Nocedal, J.: Optimization methods for large-scale
  machine learning. SIAM review  \textbf{60}(2),  223--311 (2018)

\bibitem{demoment1989image}
Demoment, G.: Image reconstruction and restoration: Overview of common
  estimation structures and problems. IEEE Transactions on Acoustics, Speech,
  and Signal Processing  \textbf{37}(12),  2024--2036 (1989)

\bibitem{nsynth2017}
Engel, J., Resnick, C., Roberts, A., Dieleman, S., Norouzi, M., Eck, D.,
  Simonyan, K.: Neural audio synthesis of musical notes with wavenet
  autoencoders. In: Int. Conf. on Machine Learning. pp. 1068--1077. PMLR (2017)

\bibitem{gribonval2021compressive}
Gribonval, R., Blanchard, G., Keriven, N., Traonmilin, Y.: Compressive
  statistical learning with random feature moments. Mathematical Statistics and
  Learning  \textbf{3}(2),  113--164 (2021)

\bibitem{gribonval2021statistical}
Gribonval, R., Blanchard, G., Keriven, N., Traonmilin, Y.: Statistical learning
  guarantees for compressive clustering and compressive mixture modeling.
  Mathematical Statistics and Learning  \textbf{3}(2),  165--257 (2021)

\bibitem{gribonval2020sketching}
Gribonval, R., Chatalic, A., Keriven, N., Schellekens, V., Jacques, L.,
  Schniter, P.: Sketching datasets for large-scale learning (long version).
  arXiv preprint arXiv:2008.01839  (2020)

\bibitem{gribonval2021sketching}
Gribonval, R., Chatalic, A., Keriven, N., Schellekens, V., Jacques, L.,
  Schniter, P.: Sketching data sets for large-scale learning: Keeping only what
  you need. IEEE Signal Processing Magazine  \textbf{38}(5),  12--36 (2021)

\bibitem{hornik1989multilayer}
Hornik, K., Stinchcombe, M., White, H.: Multilayer feedforward networks are
  universal approximators. Neural Networks  \textbf{2}(5),  359--366 (1989)

\bibitem{hurault2022gradient}
Hurault, S., Leclaire, A., Papadakis, N.: Gradient step denoiser for convergent
  plug-and-play. In: International Conference on Learning Representations
  (2022)

\bibitem{keriven2018sketching}
Keriven, N., Bourrier, A., Gribonval, R., P{\'e}rez, P.: Sketching for
  large-scale learning of mixture models. Information and Inference
  \textbf{7}(3),  447--508 (2018)

\bibitem{kobler2021total}
Kobler, E., Effland, A., Kunisch, K., Pock, T.: Total deep variation: A stable
  regularization method for inverse problems. IEEE Transactions on Pattern
  Analysis and Machine Intelligence pp.~1--1 (2021)

\bibitem{lunz2018adversarial}
Lunz, S., {\"O}ktem, O., Sch{\"o}nlieb, C.B.: Adversarial regularizers in
  inverse problems. Advances in Neural Information Processing systems
  \textbf{31} (2018)

\bibitem{nair2010rectified}
Nair, V., Hinton, G.E.: Rectified linear units improve restricted boltzmann
  machines. In: Int. Conf. on Machine Learning (2010)

\bibitem{pan2016expressiveness}
Pan, X., Srikumar, V.: Expressiveness of rectifier networks. In: Int. Conf. on
  Machine Learning. pp. 2427--2435. PMLR (2016)

\bibitem{prost2021learning}
Prost, J., Houdard, A., Almansa, A., Papadakis, N.: Learning local
  regularization for variational image restoration. In: Int. Conf. on Scale
  Space and Variational Methods in Computer Vision. pp. 358--370. Springer
  (2021)

\bibitem{schellekens2018compressive}
Schellekens, V., Jacques, L.: Compressive classification (machine learning
  without learning). arXiv preprint arXiv:1812.01410  (2018)

\bibitem{schellekens2020compressive}
Schellekens, V., Jacques, L.: Compressive learning of generative networks.
  arXiv preprint arXiv:2002.05095  (2020)

\bibitem{shi2022sketching}
Shi, H., Traonmilin, Y., Aujol, J.F.: Compressive learning for patch-based
  image denoising. SIAM Journal on Imaging Sciences  \textbf{15}(3),
  1184--1212 (2022)

\bibitem{shi2023compressive}
Shi, H., Traonmilin, Y., Aujol, J.F.: Compressive learning of deep
  regularization for denoising. In: International Conference on Scale Space and
  Variational Methods in Computer Vision. pp. 162--174. Springer (2023)

\bibitem{venkatakrishnan2013plug}
Venkatakrishnan, S.V., Bouman, C.A., Wohlberg, B.: Plug-and-play priors for
  model based reconstruction. In: IEEE Global Conf. on Signal and Information
  Processing. pp. 945--948. IEEE (2013)

\bibitem{zoran2011learning}
Zoran, D., Weiss, Y.: From learning models of natural image patches to whole
  image restoration. In: Int. Conf. on Computer Vision. pp. 479--486. IEEE
  (2011)

\end{thebibliography}

\end{document}